%% file: aistats2019.tex
\documentclass[twoside]{article}

\usepackage[accepted]{aistats2019}

\usepackage{lmodern}
\usepackage{xr}

\usepackage[utf8]{inputenc} 
\usepackage[T1]{fontenc}    
\usepackage{hyperref}       
\usepackage{url}            
\usepackage{booktabs}       
\usepackage{amsfonts}       
\usepackage{nicefrac}       
\usepackage{microtype}      

\usepackage{stmaryrd}
\hypersetup{
    colorlinks,
    linkcolor={red},
    citecolor={blue},
    urlcolor={blue}
}

\usepackage{algorithm}
\usepackage[noend]{algorithmic}

\usepackage{float}

\usepackage{lipsum}
\usepackage{amsfonts}
\usepackage{amsmath} 
\usepackage{amsthm}
\usepackage{mathtools}
\usepackage[rgb,dvipsnames]{xcolor}
\usepackage{graphicx}
\graphicspath{{../figures/}}
\usepackage{float}

\input{macros}
\newcommand{\h}{\vec{h}}
\newcommand{\ie}{i.e.\ }
\newcommand{\eg}{e.g.\ }
\renewcommand{\e}{\vec{e}}

\usepackage[symbol]{footmisc}

\renewcommand{\thefootnote}{\fnsymbol{footnote}}
\usepackage{todonotes}

\newcommand{\ceil}[1]{\lceil #1 \rceil}

\newcommand{\rev}[1]{#1}
\usepackage[normalem]{ulem}
\begin{document}
\twocolumn[

\aistatstitle{Connecting Weighted Automata and Recurrent Neural Networks through Spectral Learning}

\aistatsauthor{ Guillaume Rabusseau\footnotemark[1]\footnotemark[2] \And Tianyu Li\footnotemark[1]\footnotemark[3] \And  Doina Precup\footnotemark[1]\footnotemark[3] }
\runningauthor{ Guillaume Rabusseau, Tianyu Li,  Doina Precup }
\aistatsaddress{ \texttt{grabus@iro.umontreal.ca} \And \texttt{tianyu.li@mail.mcgill.ca} \And \texttt{dprecup@cs.mcgill.ca}  } 
]


\begin{abstract}
In this paper, we unravel  a fundamental connection between weighted finite automata~(WFAs) and second-order
recurrent neural networks~(2-RNNs): in the case of sequences of discrete symbols, WFAs and 2-RNNs with linear activation
functions are expressively equivalent. Motivated by this result, we build upon a recent extension of the spectral learning algorithm
to vector-valued WFAs and propose the first provable learning algorithm for linear 2-RNNs defined over sequences of continuous input vectors. 
This algorithm relies on estimating low rank sub-blocks of the so-called Hankel tensor, from which the parameters of a linear 2-RNN can be provably recovered. 
The performances of the proposed method are assessed in a simulation study.
\end{abstract}

\section{Introduction}


Many tasks in natural language processing, computational biology, reinforcement learning, and time series analysis rely on learning
with sequential data, \ie estimating functions defined over sequences of observations from training data.  
Weighted  finite automata~(WFAs) and recurrent neural networks~(RNNs) are two powerful and flexible classes of models which can efficiently represent such functions.
On the one hand, WFAs are tractable, they encompass a wide range of machine learning models~(they can for example compute any probability distribution defined by a hidden Markov 
model~(HMM)~\cite{denis2008rational} and can model the transition and observation behavior of
partially observable Markov decision processes~\cite{thon2015links}) and they offer appealing theoretical guarantees. In particular, 
the so-called 
\emph{spectral
methods} for learning HMMs~\cite{hsu2009spectral}, WFAs~\cite{bailly2009grammatical,balle2014spectral} and related models~\cite{glaude2016pac,boots2011closing}, 
provide an alternative to Expectation-Maximization based algorithms that is both computationally efficient and 
consistent. 
On the other hand, RNNs are 
remarkably expressive models --- they can represent any computable function~\cite{siegelmann1992computational} --- and they have successfully
tackled many practical problems in speech and audio recognition~\cite{graves2013speech,mikolov2011extensions,gers2000learning}, but
their theoretical  analysis is difficult. Even though recent work provides interesting results on their
expressive power~\cite{khrulkov2018expressive,yu2017long} as well as alternative training algorithms coming with learning guarantees~\cite{sedghi2016training},
the theoretical understanding of RNNs is still limited.

\footnotetext{\footnotemark[1] Mila \footnotemark[2] Université de Montréal \footnotemark[3] McGill University}

\renewcommand*{\thefootnote}{\arabic{footnote}}

In this work, we bridge a gap between these two classes of models by unraveling a fundamental connection between WFAs and second-order RNNs~(2-RNNs): 
\textit{when considering input sequences of discrete symbols, 2-RNNs with linear activation functions and WFAs are one and the same}, \ie they are expressively
equivalent and there exists a one-to-one mapping between the two classes~(moreover, this mapping conserves model sizes).  While connections between
finite state machines~(\eg deterministic finite automata)  and recurrent neural networks have been noticed and investigated in the past~(see \eg \cite{giles1992learning,omlin1996constructing}), to the
best of our knowledge this is the first time that such a \rev{rigorous} equivalence between linear 2-RNNs and \emph{weighted} automata is explicitly formalized. 
\rev{More precisely, we pinpoint exactly the class of recurrent neural architectures to which weighted automata are equivalent, namely second-order RNNs with
linear activation functions.}
This result naturally leads to the observation that linear 2-RNNs are a natural generalization of WFAs~(which take sequences of \emph{discrete} observations as
inputs) to sequences of \emph{continuous vectors}, and raises the question of whether the spectral learning algorithm for WFAs can be extended to linear 2-RNNs. 
The second contribution of this paper is to show that the answer is in the positive: building upon the spectral learning algorithm for vector-valued WFAs introduced
recently in~\cite{rabusseau2017multitask}, \emph{we propose the first provable learning algorithm for second-order RNNs with linear activation functions}.
Our learning algorithm relies on estimating  sub-blocks of the so-called Hankel tensor, from which the parameters of a 2-linear RNN can be recovered
using basic linear algebra operations. One of the key technical difficulties in designing this algorithm resides in estimating
these sub-blocks from training data where the inputs are sequences of \emph{continuous} vectors. 
We leverage multilinear properties of linear 2-RNNs and the fact that the Hankel sub-blocks can be reshaped into higher-order tensors of low tensor train rank~(a
result we believe is of independent interest) to perform this estimation efficiently using matrix sensing and tensor recovery techniques.
As a proof of concept, we validate our theoretical findings in a simulation study on toy examples where we experimentally compare  
different recovery methods and investigate the robustness of our algorithm to noise and rank mis-specification. \rev{We also show that refining the estimator returned
by our algorithm using stochastic gradient descent can lead to significant improvements.}

\rev{
\paragraph{Summary of contributions.} We formalize a \emph{strict equivalence between weighted automata and second-order RNNs with linear activation
functions}~(Section~\ref{sec:WFAs.and.2RNNs}), showing that linear 2-RNNs can be seen as a natural extension of (vector-valued) weighted automata for input sequences of \emph{continuous} vectors. We then
propose a \emph{consistent learning algorithm for linear 2-RNNs}~(Section~\ref{sec:Spectral.learning.of.2RNNs}).
The relevance of our contributions can be seen from two perspectives.
First, while learning feed-forward neural networks with linear activation functions is a trivial task (it reduces to linear or reduced-rank regression), this
is not at all the case for recurrent architectures with linear activation functions; to the best of our knowledge, our algorithm is the \emph{first consistent learning algorithm
for the class of functions computed by linear second-order recurrent networks}. Second, from the perspective of learning weighted automata, we propose a  natural extension of WFAs to continuous inputs and \emph{our learning algorithm addresses the long-standing limitation of the spectral learning method to discrete inputs}.
}

\paragraph{Related work.}
Combining the spectral learning algorithm for WFAs with matrix completion techniques~(a problem which is closely related to matrix sensing) has
been theoretically investigated in~\cite{balle2012spectral}. An extension of probabilistic transducers to continuous inputs~(along with a spectral learning algorithm) has been proposed in~\cite{recasens2013spectral}.
The connections between tensors  and RNNs have been previously leveraged to study the expressive power of RNNs in~\cite{khrulkov2018expressive}
and to achieve model compression in~\cite{yu2017long,yang2017tensor,tjandra2017compressing}. 
Exploring relationships between RNNs and automata has recently received a renewed interest~\cite{peng2018rational,chen2018recurrent,li2018nonlinear}. In particular, such connections have been explored for interpretability purposes~\cite{weiss2018extracting,ayache2018explaining} and the ability of RNNs to learn classes of formal languages
has been investigated in~\cite{avcu2017subregular}.  \rev{Connections between the tensor train decomposition and WFAs have been previously noticed in~\cite{critch2013algebraic,critch2014algebraic,rabusseau2016thesis}.} 
The predictive state RNN model introduced in~\cite{downey2017predictive} is closely related to 2-RNNs and the authors propose
to use the spectral learning algorithm for predictive state representations to initialize a gradient based algorithm; their approach however comes without
theoretical guarantees. Lastly, a provable algorithm for RNNs relying on the tensor method of moments has been proposed in~\cite{sedghi2016training} but
it is limited to first-order RNNs with quadratic activation functions~(which do not encompass linear 2-RNNs).

\emph{The proofs of the results given in the paper can be found in the supplementary material.}





\section{Preliminaries}\label{sec:prelim}
In this section, we first present basic notions of tensor algebra before introducing  second-order recurrent neural 
network, 
weighted finite automata and the spectral learning algorithm.
We start by introducing some notation.
For any integer $k$ we use $[k]$ to denote the set of integers from $1$ to $k$. We use
$\lceil l \rceil$ to denote the smallest integer greater or equal to $l$.
For any set $\Scal$, we denote by $\Scal^*=\bigcup_{k\in\Nbb}\Scal^k$ the set of all
finite-length sequences of elements of $\Scal$~(in particular, 
$\Sigma^*$ will denote the set of strings on a finite alphabet $\Sigma$). 
We use lower case bold letters  for vectors (\eg $\vec{v} \in \Rbb^{d_1}$),
upper case bold letters for matrices (\eg $\M \in \Rbb^{d_1 \times d_2}$) and
bold calligraphic letters for higher order tensors (\eg $\T \in \Rbb^{d_1
\times d_2 \times d_3}$). We use $\e_i$ to denote the $i$th canonical basis 
vector of $\R^d$~(where the dimension $d$ will always appear clearly from context).
The $d\times d$ identity matrix will be written as $\I_d$.
The $i$th row (resp. column) of a matrix $\M$ will be denoted by
$\M_{i,:}$ (resp. $\M_{:,i}$). This notation is extended to
slices of a tensor in the straightforward way.
If $\vec{v} \in \Rbb^{d_1}$ and $\vec{v}' \in \Rbb^{d_2}$, we use $\vec{v} \kron \vec{v}' \in \Rbb^{d_1
\cdot d_2}$ to denote the Kronecker product between vectors, and its
straightforward extension to matrices and tensors.
Given a matrix $\M \in \Rbb^{d_1 \times d_2}$, we use $\vectorize{\M} \in \Rbb^{d_1
\cdot d_2}$ to denote the column vector obtained by concatenating the columns of
$\M$. The inverse of $\M$ is denoted by $\M\inv$, its Moore-Penrose pseudo-inverse
by $\M\pinv$, and the transpose of its inverse by $\M\invtop$; the Frobenius norm
is denoted by $\norm{\M}_F$ and the nuclear norm by $\norm{\M}_*$.

\paragraph{Tensors.}
We first recall basic definitions of tensor algebra; more details can be found
in~\cite{Kolda09}. 
A \emph{tensor} $\T\in \Rbb^{d_1\times\cdots \times d_p}$ can simply be seen
as a multidimensional array $(\T_{i_1,\cdots,i_p}\ : \ i_n\in [d_n], n\in [p])$. The
\emph{mode-$n$} fibers of $\T$ are the vectors obtained by fixing all
indices except  the $n$th one, \eg $\T_{:,i_2,\cdots,i_p}\in\Rbb^{d_1}$.
The \emph{$n$th mode matricization} of $\T$ is the matrix having the
mode-$n$ fibers of $\T$ for columns and is denoted by
$\tenmat{T}{n}\in \Rbb^{d_n\times d_1\cdots d_{n-1}d_{n+1}\cdots d_p}$.
The vectorization of a tensor is defined by $\vectorize{\T}=\vectorize{\tenmat{T}{1}}$.
In the following $\T$ always denotes a tensor of size $d_1\times\cdots \times d_p$.

The \emph{mode-$n$ matrix product} of the tensor $\T$ and a matrix
$\X\in\Rbb^{m\times d_n}$ is a tensor  denoted by $\T\ttm{n}\X$. It is 
of size $d_1\times\cdots \times d_{n-1}\times m \times d_{n+1}\times
\cdots \times d_p$ and is defined by the relation 
$\Y = \T\ttm{n}\X \Leftrightarrow \tenmat{Y}{n} = \X\tenmat{T}{n}$.
The \emph{mode-$n$ vector product} of the tensor $\T$ and a vector
$\vec{v}\in\Rbb^{d_n}$ is a tensor defined by $\T\ttv{n}\vec{v} = \T\ttm{n}\vec{v}^\top
\in \Rbb^{d_1\times\cdots \times d_{n-1}\times d_{n+1}\times
\cdots \times d_p}$.
%
%
It is easy to check that the $n$-mode product satisfies $(\T\ttm{n}\mat{A})\ttm{n}\mat{B} = \T\ttm{n}\mat{BA}$
where we assume compatible dimensions of the tensor $\T$ and
the matrices $\A$ and $\B$.

Given strictly positive integers $n_1,\cdots, n_k$ satisfying
$\sum_i n_i = p$, we use the notation $\tenmatgen{\T}{n_1,n_2,\cdots,n_k}$ to denote the $k$th order tensor 
obtained by reshaping $\T$ into a tensor\footnote{Note that the specific ordering used to perform matricization, vectorization
and such a reshaping is not relevant as long as it is consistent across all operations.} of size 
$(\prod_{i_1=1}^{n_1} d_{i_1}) \times (\prod_{i_2=1}^{n_2} d_{n_1 + i_2}) \times \cdots \times (\prod_{i_k=1}^{n_k} d_{n_1+\cdots+n_{k-1} + i_k})$.
In particular we have $\tenmatgen{\T}{p} = \vectorize{\T}$ and $\tenmatgen{\T}{1,p-1} = \tenmat{\T}{1}$.

A rank $R$ \emph{tensor train (TT) decomposition}~\cite{oseledets2011tensor} of a tensor 
$\T\in\R^{d_1\times \cdots\times d_p}$ consists in factorizing $\T$ into the product of $p$ core tensors
$\G_1\in\R^{d_1\times R},\G_2\in\R^{R\times d_2\times R},
\cdots, \G_{p-1}\in\R^{R\times d_{p-1} \times R},
\G_p \in \R^{R\times d_p}$, and is defined\footnote{The classical definition of the TT-decomposition allows the rank $R$ to be different
for each mode, but this definition is sufficient for the purpose of this paper.} by
\begin{align*}
\MoveEqLeft\T_{i_1,\cdots,i_p} =  
&(\G_1)_{i_1,:}(\G_2)_{:,i_2,:}\cdots 
 (\G_{p-1})_{:,i_{p-1},:}(\G_p)_{:,i_p}
\end{align*}
for all indices $i_1\in[d_1],\cdots,i_p\in[d_p]$; we will use the notation $\T = \TT{\G_1,\cdots,\G_p}$
to denote such a decomposition. A tensor network representation of this decomposition is shown in Figure~\ref{fig:tn.TT}.
 While the problem of finding the best approximation of  TT-rank $R$
of a given tensor is NP-hard~\cite{hillar2013most}, 
a quasi-optimal SVD based compression algorithm~(TT-SVD) has been proposed 
in~\cite{oseledets2011tensor}.
It is worth mentioning that the TT decomposition is invariant under change of basis: 
for any invertible matrix $\M$ and any core tensors $\G_1,\G_2,\cdots,\G_p$, we have
$\TT{\G_1,\cdots,\G_p} = \TT{\G_1\ttm{2}\M\invtop,\G_2\ttm{1}\M\ttm{3}\M\invtop,\cdots,
\G_{p-1}\ttm{1}\M\ttm{3}\M\invtop,\G_p\ttm{1}\M}$.


\begin{figure}
\begin{center}
\resizebox{0.45\textwidth}{0.08\textwidth}{%
\begin{tikzpicture}
	\input{tikz_tensor_networks}
	\node[tensor](G1){$\ten{G}_1$};
	\node[draw = none,below=0.8cm of G1](G11){};
	
	\node[tensor,right = 1cm of G1](G2){$\ten{G}_2$};
	\node[draw=none,below=0.8cm of G2](G22){};
	
	\node[tensor,right = 1cm of G2](G3){$\ten{G}_3$};
	\node[draw=none,below=0.8cm of G3](G32){};
	
	\node[tensor,right = 1cm of G3](G4){$\ten{G}_4$};
	\node[draw=none,below=0.8cm of G4](G41){};
	
	\node[tensor,left=3cm of G1](T){$\T$};
	\node[draw=none,below left = 0.2cm and 1cm of T](T1){};
	\node[draw=none,below left = 0.8cm and 0.2cm of T](T2){};
	\node[draw=none,below right = 0.8cm and 0.2cm of T](T3){};
	\node[draw=none,below right = 0.2cm and 1cm of T](T4){};
	\node[draw=none,right=1.5cm of T](eq){$=$};	
		
	\edgeports{T}{1}{above left}{T1}{}{}{$d_1$};
	\edgeports{T}{2}{below right}{T2}{}{}{$d_2$};
	\edgeports{T}{3}{below right=-0.1cm and 0.1cm}{T3}{}{}{$d_3$};
	\edgeports{T}{4}{above right}{T4}{}{}{$d_4$};
	
	\edgeports{G1}{1}{below left = -0.1cm and 0.01cm }{G11}{}{}{$d_1$};
	
	\edgeports{G2}{1}{below left}{G1}{2}{below right}{$R$};	
	\edgeports{G2}{2}{below left = -0.1cm and 0.01cm }{G22}{}{}{$d_2$};
	
	\edgeports{G3}{1}{below left}{G2}{3}{below right}{$R$};	
	\edgeports{G3}{2}{below left = -0.1cm and 0.01cm }{G32}{}{}{$d_3$};
	
	\edgeports{G4}{1}{below left}{G3}{3}{below right}{$R$};	
	\edgeports{G4}{2}{below left = -0.1cm and 0.01cm }{G41}{}{}{$d_4$};

\end{tikzpicture}
}%
\end{center}
\caption{Tensor network representation of a rank $R$ tensor train decomposition~(nodes represent tensors and an edge between two nodes
represents a contraction between the corresponding modes of the two tensors).}
\label{fig:tn.TT}
\end{figure}
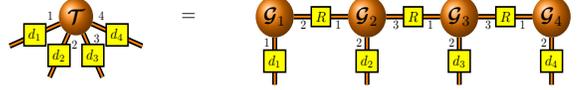

\paragraph{Second-order RNNs.}
A \emph{second-order recurrent neural network} (2-RNN)~\cite{giles1990higher,pollack1991induction,lee1986machine}\footnote{Second-order reccurrent architectures have also been successfully used more recently, see \eg \cite{sutskever2011generating} and \cite{wu2016multiplicative}.} with $n$ hidden units can be defined as a tuple 
$M=(\h_0,\Aten,\vvsinf)$ where $\h_0\in\R^n$ is the initial state, $\Aten\in\R^{ n\times d \times n }$ is the transition tensor, and
$\vvsinf\in \R^{p\times n}$ is the output matrix,
with $d$ and $p$ being the input and output dimensions respectively. 
A 2-RNN maps any sequence of inputs $\x_1,\cdots,\x_k\in\R^d$ to
a sequence of outputs $\y_1,\cdots,\y_k\in\R^p$ defined for any $ t=1,\cdots,k$ by
\begin{equation}\label{eq:2RNN.definition}
\y_t = z_2(\vvsinf\h_t) \text{ with }\h_t = z_1(\Aten\ttv{1}\x_t\ttv{2}\h_{t-1})
\end{equation}
where $z_1:\R^n\to\R^n$ and $z_2:\R^p\to\R^p$ are activation functions.
Alternatively, one can think of a 2-RNN as computing
a function $f_M:(\R^d)^*\to\R^p$ mapping each input sequence $\x_1,\cdots,\x_k$ to the corresponding final output $\y_k$.
While $z_1$ and $z_2$ are usually non-linear
component-wise functions, we consider in this paper the case where both $z_1$ and $z_2$ are the identity, and we refer to
the resulting model as a \emph{linear 2-RNN}.
For a linear 2-RNN $M$, the function $f_M$ is multilinear in the sense that, for any integer $l$, its restriction to the domain $(\R^d)^l$ is
multilinear. Another useful observation is that linear 2-RNNs are invariant under change of basis: for any invertible matrix
$\P$, the linear 2-RNN $\tilde{M}=(\P\invtop\h_0,\Aten\ttm{1}\P\ttm{3}\P\invtop,\P\vvsinf)$ is such that $f_{\tilde{M}}=f_M$. A linear 2-RNN $M$ with $n$ states is called \emph{minimal} if its number of hidden units is minimal~(\ie any linear 2-RNN computing $f_M$ has
at least $n$ hidden units).

\paragraph{Weighted automata and spectral learning.} \emph{Vector-valued weighted finite automaton}~(vv-WFA) have
been introduced in~\cite{rabusseau2017multitask} as a natural generalization of weighted automata from scalar-valued functions
to vector-valued ones. A $p$-dimensional vv-WFA with $n$ states is a tuple $A=\vvwa$
where $\szero\in\R^n$ is the initial weights vector, $\vvsinf\in\R^{p\times n}$ is the matrix of final weights, 
and $\A^\sigma\in\R^{n\times n}$ is the transition matrix for each symbol $\sigma$ in a finite alphabet $\Sigma$.
A vv-WFA $A$ computes a function $f_A:\Sigma^*\to\R^p$ defined by 
$$f_A(x) =\vvsinf (\A^{x_1}\A^{x_2}\cdots\A^{x_k})^\top\szero $$
for each word $x=x_1x_2\cdots x_k\in\Sigma^*$. We call a vv-WFA \emph{minimal} if its number of states
is minimal. Given a function $f:\Sigma^*\to \Rbb^p$ we denote by $\rank(f)$ the number of states of a minimal vv-WFA computing $f$~(which is
set to $\infty$ if $f$ cannot be computed by a vv-WFA).

The  spectral learning algorithm for  vv-WFAs relies on the following fundamental theorem relating
the rank of a function $f:\Sigma^*\to\R^d$ to its Hankel tensor $\Hten \in \R^{\Sigma^*\times\Sigma^*\times p}$, which is defined
by $\Hten_{u,v,:} = f(uv)$ for all $u,v\in\Sigma^*$.
\begin{theorem}[\cite{rabusseau2017multitask}]
\label{thm:fliess-vvWFA}
Let $f:\Sigma^*\to\R^d$ and let $\Hten$ be its Hankel tensor. Then $\rank(f) = \rank(\tenmat{H}{1})$.
\end{theorem}
The vv-WFA learning algorithm leverages the fact that the proof of this theorem is constructive: one can recover a vv-WFA computing $f$
from any low rank factorization of $\tenmat{H}{1}$. In practice, a finite sub-block $\Hten_{\Pcal,\Scal} \in \R^{\Pcal\times \Scal\times p}$ of the Hankel tensor  is used
to recover the vv-WFA, where $\Pcal,\Scal\subset\Sigma^*$ are finite sets of prefixes and suffixes forming a \emph{complete basis} for $f$, \ie such that
$\rank(\tenmatpar{\Hten_{\Pcal,\Scal}}{1}) = \rank(\tenmat{H}{1})$. More details can be found 
in~\cite{rabusseau2017multitask}.

\section{A Fundamental Relation between WFAs and Linear 2-RNNs}\label{sec:WFAs.and.2RNNs}

We start by unraveling a fundamental connection between vv-WFAs and linear 2-RNNs: vv-WFAs and
linear 2-RNNs are expressively equivalent for representing functions defined over sequences of
discrete symbols. \rev{Moreover, both models have the same capacity in the sense that there is a direct
correspondence between the hidden units of a linear 2-RNN and the states of a vv-WFA computing the same function}. More formally, we have the following theorem. 
\begin{theorem}\label{thm:2RNN-vvWFA}
Any function that can be computed by a vv-WFA with $n$ states can be computed by a linear 2-RNN with $n$ hidden units.
Conversely, any function that can be computed by a linear 2-RNN with $n$ hidden units on sequences of one-hot vectors~(\ie canonical basis 
vectors) can be computed by a WFA with $n$ states.

More precisely, the WFA $A=\vvwa$ with $n$ states and the linear 2-RNN $M=(\szero,\Aten,\vvsinf)$ with
$n$ hidden units, where $\Aten\in\R^{n\times \Sigma \times n}$ is defined by $\Aten_{:,\sigma,:}=\A^\sigma$ for all $\sigma\in\Sigma$, are
such that
$f_A(\sigma_1\sigma_2\cdots\sigma_k) = f_M(\x_1,\x_2,\cdots,\x_k)$ for all sequences of input symbols $\sigma_1,\cdots,\sigma_k\in\Sigma$,
where for each $i\in[k]$ the input vector $\x_i\in\R^\Sigma$ is
the one-hot encoding of the symbol $\sigma_i$.

\end{theorem}

This result first implies that linear 2-RNNs defined over sequence of discrete symbols~(using one-hot encoding) \emph{can be provably learned using the spectral
learning algorithm for WFAs/vv-WFAs}; indeed, these  algorithms have been proved to return consistent estimators.
\rev{Let us stress again that, contrary to the case of feed-forward architectures, learning recurrent networks with linear activation functions is not a trivial task.}
Furthermore, Theorem~\ref{thm:2RNN-vvWFA} reveals that linear 2-RNNs are a natural generalization of classical weighted automata  to functions
defined over sequences of continuous vectors~(instead of discrete symbols). This spontaneously raises the question of whether the spectral learning algorithms
for WFAs and vv-WFAs can be extended to the general setting of linear 2-RNNs; we show that the answer is in the positive in the next section.

\section{Spectral Learning of Linear 2-RNNs}\label{sec:Spectral.learning.of.2RNNs}
In this section, we extend the learning algorithm for vv-WFAs to linear 2-RNNs, \rev{thus at the same time addressing the limitation of the spectral learning algorithm to discrete inputs and  providing the first consistent learning algorithm for linear second-order RNNs.}


\subsection{Recovering 2-RNNs from Hankel Tensors}\label{subsec:SL-2RNN}
We first present an identifiability result showing how one can recover a linear 2-RNN computing a function $f:(\R^d)^*\to \R^p$ from observable tensors extracted from some Hankel tensor
associated with $f$. Intuitively, we obtain this result by reducing the problem to the one of learning a vv-WFA. This is done by considering the restriction of
$f$ to canonical basis vectors; loosely speaking, since the domain of this restricted function is isomorphic to $[d]^*$, this allows us to fall back onto the 
setting of sequences of discrete symbols.

Given a function $f:(\R^d)^*\to\R^p$, we define its Hankel tensor $\Hten_f\in \R^{[d]^* \times [d]^* \times p}$ by
$$(\Hten_f)_{i_1\cdots i_s,  j_1\cdots j_t,:} = f(\e_{i_1},\cdots,\e_{i_s},\e_{j_1},\cdots,\e_{j_t}),$$ 
for all $i_1,\cdots,i_s,j_1,\cdots,j_t\in [d]$, which is infinite in two
of its modes. It is easy to see that $\Hten_f$ is also the Hankel tensor associated with the function $\tilde{f}:[d]^* \to \R^p$ mapping any
sequence $i_1i_2\cdots i_k\in[d]^*$ to $f(\e_{i_1},\cdots,\e_{i_k})$. Moreover, in the special case where $f$ can be computed by a linear 2-RNN, one
can use the multilinearity of $f$ to show that $f(\x_1,\cdots,\x_k) = \sum_{i_1,\cdots,i_k = 1}^d (\x_1)_{i_1}\cdots(\x_l)_{i_k} \tilde{f}(i_1\cdots i_k)$,
\rev{giving us some intuition on how one could} learn $f$ by learning a vv-WFA computing $\tilde{f}$ using the spectral learning algorithm. 
That is, given a large enough sub-block $\Hten_{\Pcal,\Scal}\in \R^{\Pcal\times \Scal\times p}$ of $\Hten_f$ for some prefix and suffix sets $\Pcal,\Scal\subseteq [d]^*$, 
one should be able to recover a vv-WFA computing $\tilde{f}$ and consequently a linear 2-RNN computing $f$ using Theorem~\ref{thm:2RNN-vvWFA}. 
\rev{Before devoting the remaining of this section to formalize this intuition~(leading to Theorem~\ref{thm:2RNN-SL}), it is worth 
observing that while this approach is sound, it is not realistic since it requires observing entries of the Hankel tensor $\Hten_f$, which implies having access to input/output examples where the inputs are  \emph{sequences of canonical basis vectors}; This issue will be discussed in more details and addressed in the 
next section.} 

\emph{For the sake of clarity, we present the learning algorithm for the particular case where there exists an $L$ such that 
the prefix and suffix sets consisting of all sequences of length $L$, that is $\Pcal = \Scal = 
[d]^L$, forms a complete basis for $\tilde{f}$}~\rev{(\ie the sub-block $\Hten_{\Pcal,\Scal}\in\R^{[d]^L\times [d]^L\times p}$ of the Hankel tensor $\Hten_f$ is
such that $\rank(\tenmatpar{\Hten_{\Pcal,\Scal}}{1}) = \rank(\tenmatpar{\Hten_f}{1})$)}. This assumption allows us to present all the key elements of the algorithm  in a simpler way, the technical details
needed to lift this assumption are given in the supplementary material.

For any integer $l$, we define the finite tensor 
$\Hten^{(l)}_f\in \R^{ d\times \cdots \times d\times p}$ of order $l+1$ by
$$  (\Hten^{(l)}_f)_{i_1,\cdots,i_l,:} = f(\e_{i_1},\cdots,\e_{i_l}) \ \ \ \text{for all } i_1,\cdots,i_l\in [d].$$ 
Observe that for any integer $l$, the tensor $\Hten^{(l)}_f$ can be obtained by reshaping a finite sub-block of the Hankel tensor $\Hten_f$. 
When $f$ is computed by a linear $2$-RNN, we have the useful property that, for any integer $l$,
\begin{equation}\label{eq:}
 f(\x_1,\cdots,\x_l)  = \Hten^{(l)}_f \ttv{1} \x_1 \ttv{2} \cdots \ttv{l} \x_l
\end{equation}
for any sequence of inputs $\x_1,\cdots,\x_l\in\R^d$~(which can be shown using the
multilinearity of $f$). Another fundamental property of the tensors $\Hten^{(l)}_f$ is that they are of low tensor train rank. Indeed, for any $l$, one can check that
$\Hten^{(l)}_f = \TT{\Aten\ttv{1}\szero, \underbrace{\Aten, \cdots, \Aten}_{l-1\text{ times}}, \vvsinf^\top}$~(the tensor network representation of this
decomposition is shown in Figure~\ref{fig:Hl.TT.rank}).
This property will be particularly relevant to the learning algorithm we design in the following section, but it is also a fundamental relation
that deserves some attention on its own: it implies in particular that, beyond the classical relation between the rank of the Hankel matrix $\H_f$ and the number states of
a minimal WFA computing $f$, the Hankel matrix possesses a deeper structure intrinsically connecting weighted automata to the tensor
train decomposition. 
\rev{We now state the main result of this section, showing that a (minimal) linear 2-RNN computing a function $f$ can be exactly recovered from sub-blocks of the Hankel tensor $\Hten_f$.}

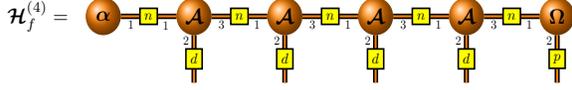
\begin{figure}
\begin{center}
\resizebox{0.45\textwidth}{0.08\textwidth}{%
\begin{tikzpicture}
	\input{tikz_tensor_networks}
	\node[tensor](G1){$\szero$};
	\node[draw=none,left of = G1](H){$\Hten^{(4)}_f=\ \ \ $};
	\node[tensor,right = 1cm of G1](G2){$\ten{A}$};
	\node[draw=none,below=0.8cm of G2](G22){};
	
	\node[tensor,right = 1cm of G2](G3){$\ten{A}$};
	\node[draw=none,below=0.8cm of G3](G32){};

	\node[tensor,right = 1cm of G3](G4){$\ten{A}$};
	\node[draw=none,below=0.8cm of G4](G42){};	
	
	\node[tensor,right = 1cm of G4](G4-){$\ten{A}$};
	\node[draw=none,below=0.8cm of G4-](G42-){};	
	
	\node[tensor,right = 1cm of G4-](G5){$\vvsinf$};
	\node[draw=none,below=0.8cm of G5](G51){};

	\edgeports{G2}{1}{below left = 0.01cm }{G1}{1}{below right}{$n$};	
	\edgeports{G2}{2}{below left = -0.1cm and 0.01cm }{G22}{}{}{$d$};
	
	\edgeports{G3}{1}{below left}{G2}{3}{below right}{$n$};	
	\edgeports{G3}{2}{below left = -0.1cm and 0.01cm }{G32}{}{}{$d$};
	
	\edgeports{G4}{1}{below left}{G3}{3}{below right}{$n$};	
	\edgeports{G4}{2}{below left = -0.1cm and 0.01cm }{G42}{}{}{$d$};
	
	\edgeports{G4-}{1}{below left}{G4}{3}{below right}{$n$};	
	\edgeports{G4-}{2}{below left = -0.1cm and 0.01cm }{G42-}{}{}{$d$};
	
	\edgeports{G5}{1}{below left}{G4-}{3}{below right}{$n$};	
	\edgeports{G5}{2}{below left = -0.1cm and 0.01cm }{G51}{}{}{$p$};

\end{tikzpicture}
}%
\end{center}
\caption{Tensor network representation of the TT decomposition of the Hankel tensor $\Hten^{(4)}_f$ induced
by a linear $2$-RNN $(\szero,\Aten,\vvsinf)$.}
\label{fig:Hl.TT.rank}
\end{figure}

\begin{theorem}\label{thm:2RNN-SL}
Let $f:(\R^d)^*\to \R^p$ be a function computed by a minimal linear $2$-RNN  with $n$ hidden units and let
$L$ be an integer such that $\rank(\tenmatgen{\Hten^{(2L)}_f}{L,L+1}) = n$.

Then, for any $\P\in\R^{d^L\times n}$ and $\S\in\R^{n\times d^Lp}$ such that $\tenmatgen{\Hten^{(2L)}_f}{L,L+1} = \P\S$, the
linear 2-RNN $M=(\szero,\Aten,\vvsinf)$ defined by

$$\szero = (\S\pinv)^\top\tenmatgen{\Hten^{(L)}_f}{L+1}, \ \ \ \ \vvsinf^\top = \P\pinv\tenmatgen{\Hten^{(L)}_f}{L,1}
$$
$$\Aten = (\tenmatgen{\Hten^{(2L+1)}_f}{L,1,L+1})\ttm{1}\P\pinv\ttm{3}(\S\pinv)^\top$$
is a minimal linear $2$-RNN computing $f$.
\end{theorem}
First observe that such an integer $L$ exists under the assumption that $\Pcal = \Scal = 
[d]^L$ forms a complete basis for $\tilde{f}$.
It is also worth mentioning that a necessary condition for $\rank(\tenmatgen{\Hten^{(2L)}_f}{L,L+1}) = n$ is that
$d^L\geq n$, \ie $L$ must be of the order $\log_d(n)$.

\subsection{Hankel Tensors Recovery from Linear Measurements}\label{subsec:Hankel.tensor.recovery}

We showed in the previous section that, given the Hankel tensors $\Hten^{(L)}_f$, $\Hten^{(2L)}_f$ and $\Hten^{(2L+1)}_f$, one can recover 
a linear 2-RNN computing $f$ if it exists. This first implies that the class of functions that can be computed by linear 2-RNNs is learnable  in Angluin's
exact learning model~\cite{angluin1988queries} where one has access to an oracle that can answer membership queries~(\eg \textit{what is the value computed by the target $f$ 
on~$(\x_1,\cdots,\x_k)$?}) and  equivalence queries~(\eg \textit{is my current hypothesis $h$ equal to the target $f$?}). While this fundamental result is 
of significant theoretical interest, assuming access to such an oracle is unrealistic. In this section, we show that a stronger learnability result can be obtained in a more realistic setting,
where we  only
assume access to randomly generated input/output examples $((\x^{(i)}_1,\x_2^{(i)},\cdots,\x_l^{(i)}),\y^{(i)})\in(\R^d)^*\times\R^p$ where $\y^{(i)} = f(\x^{(i)}_1,\x_2^{(i)},\cdots,\x_l^{(i)})$.

The key observation is that such an input/output example $((\x^{(i)}_1,\x_2^{(i)},\cdots,\x_l^{(i)}),\y^{(i)})$ can be seen as a \emph{linear measurement} of the 
Hankel tensor $\Hten^{(l)}$. Indeed, we have
\begin{align*}
\y^{(i)} &= f(\x^{(i)}_1,\x_2^{(i)},\cdots,\x_l^{(i)}) = \Hten^{(l)}_f \ttv{1} \x_1 \ttv{2} \cdots \ttv{l} \x_l \\
&= 
\tenmatgen{\Hten^{(l)}}{l,1}^\top \x^{(i)}
\end{align*}
where $\x^{(i)} = \x^{(i)}_1\kron\cdots \kron \x_{l}^{(i)}\in\R^{d^l}$. Hence,  by regrouping $N$ output examples $\y^{(i)}$ into
the matrix $\Ymat\in\R^{N\times p}$ and the corresponding input vectors $\x^{(i)}$ into the matrix $\X\in\R^{N\times d^l}$,
one can recover $\Hten^{(l)}$ by solving the linear system $\Ymat = \X\tenmatgen{\Hten^{(l)}}{l,1}$, which has a unique
solution whenever $\X$ is of full column rank. 
This  naturally leads to the following theorem, whose proof relies on the fact that
$\X$ will be of full column rank whenever $N\geq d^l$ and  the components of each $\x^{(i)}_j$ for $j\in[l],i\in[N]$
are drawn independently from a continuous distribution over $\R^{d}$~(w.r.t. the Lebesgue measure).

%
%
%


\begin{theorem}\label{thm:learning-2RNN}
Let $(\h_0,\Aten,\vvsinf)$ be a minimal linear 2-RNN with $n$ hidden units computing a function $f:(\R^d)^*\to \R^p$, and let $L$ be an integer\footnote{Note that the theorem can be adapted if such an integer $L$ does not exists~(see supplementary material).}
such that $\rank(\tenmatgen{\Hten^{(2L)}_f}{L,L+1}) = n$.
Suppose we have access to $3$ datasets
$D_l = \{((\x^{(i)}_1,\x_2^{(i)},\cdots,\x_l^{(i)}),\y^{(i)}) \}_{i=1}^{N_l}\subset(\R^d)^l\times \R^p$ for $l\in\{L,2L,2L+1\}$ 
where the entries of each $\x^{(i)}_j$ are drawn independently from the standard normal distribution and where each
$\y^{(i)} = f(\x^{(i)}_1,\x_2^{(i)},\cdots,\x_l^{(i)})$.

Then, if $N_l \geq d^l$ for $l =L,\ 2L,\ 2L+1$,
the linear 2-RNN $M$ returned by Algorithm~\ref{alg:2RNN-SL} with the least-squares method satisfies $f_M = f$ with probability one.
\end{theorem}

\begin{algorithm}[tb]
   \caption{\texttt{2RNN-SL}: Spectral Learning of linear 2-RNNs }
   \label{alg:2RNN-SL}
\begin{algorithmic}[1]
   \REQUIRE Three training datasets $D_L,D_{2L},D_{2L+1}$ with input sequences of length $L$, $2L$ and $2L+1$ respectively, a \texttt{recovery\_method}, rank $R$ and learning rate $\gamma$~(for IHT/TIHT).
   \FOR{$l\in\{L,2L,2L+1\}$}
   \STATE\label{alg.firstline.forloop} Use $D_l = \{((\x^{(i)}_1,\x_2^{(i)},\cdots,\x_l^{(i)}),\y^{(i)}) \}_{i=1}^{N_l}\subset(\R^d)^l\times \R^p$ to build $\X\in\R^{N_l\times d^l}$ with rows $\x^{(i)}_1\kron\x_2^{(i)}\kron\cdots\kron\x_l^{(i)}$ for $i\in[N_l]$ and  $\Ymat\in\R^{N_l\times p}$ with rows $\y^{(i)}$ for $i\in[N_l]$.
   \IF{\texttt{recovery\_method} = "Least-Squares"}
   \STATE\label{alg.line.lst-sq} $\Hten^{(l)} = \displaystyle\argmin_{\T\in \R^{d\times\cdots\times d\times p}} \norm{\X\tenmatgen{\T}{l,1} - \Ymat}_F^2$.
   \ELSIF{\texttt{recovery\_method} = "Nuclear Norm"}
   \STATE $\Hten^{(l)} = \displaystyle\argmin_{\T\in \R^{d\times\cdots\times d\times p}} \norm{\tenmatgen{\T}{\ceil{l/2},l-\ceil{l/2} + 1}}_*$ subject to $\X \tenmatgen{\T}{l,1} = \Ymat$.
   \label{alg.line.nucnorm}
   \ELSIF{\texttt{recovery\_method} = "(T)IHT"}
   \STATE Initialize $\Hten^{(l)} \in \R^{d\times\cdots\times d\times p}$ to $\mat{0}$.
   \REPEAT\label{alg.line.iht.start}
   \STATE\label{alg.line.iht.gradient} $\tenmatgen{\Hten^{(l)}}{l,1} = \tenmatgen{\Hten^{(l)} }{l,1} + \gamma\X^\top(\Ymat - \X\tenmatgen{\Hten^{(l)} }{l,1})$
   \STATE $\Hten^{(l)} = \texttt{project}(\Hten^{(l)},R)$~(using either SVD for IHT or TT-SVD for TIHT)
   \UNTIL{convergence}\label{alg.line.iht.end}
   \ENDIF\label{alg.lastline.forloop} 
   \ENDFOR
   \STATE\label{alg.line.svd} Let $\tenmatgen{\Hten^{(2L)}}{L,L+1} = \P\S$ be a rank $R$ factorization.
   \STATE Return the linear 2-RNN $(\h_0,\Aten,\vvsinf)$ where 
   \begin{align*}
    \szero\ &= (\S\pinv)^\top\tenmatgen{\Hten^{(L)}_f}{L+1},\ \ \ \ \vvsinf^\top = \P\pinv\tenmatgen{\Hten^{(L)}_f}{L,1}\\
    \Aten\ &= (\tenmatgen{\Hten^{(2L+1)}_f}{L,1,L+1})\ttm{1}\P\pinv\ttm{3}(\S\pinv)^\top
\end{align*}  
\end{algorithmic}
\end{algorithm}

A few remarks on this theorem are in order.   The first observation is that the $3$
datasets $D_L$, $D_{2L}$ and $D_{2L+1}$ can either be drawn independently or not~(\eg the sequences in $D_{L}$ can
be prefixes of the sequences in $D_{2L}$ but it is not necessary). In particular, the result still holds when  the datasets $D_l$ are constructed from a unique dataset 
$S =\{((\x^{(i)}_1,\x_2^{(i)},\cdots,\x_T^{(i)}),(\y^{(i)}_1,\y^{(i)}_2,\cdots,\y^{(i)}_T)) \}_{i=1}^{N}$
of input/output sequences with $T\geq 2L+1$, where $\y^{(i)}_t = f(\x^{(i)}_1,\x_2^{(i)},\cdots,\x_t^{(i)})$ for any $t\in[T]$.
Observe that having access to such input/output training sequences is not an unrealistic assumption: for example when training RNNs for
language modeling the output $\y_t$ is the  conditional probability vector of the next symbol, and  for classification tasks the output is
the one-hot encoded label for all time steps. Lastly, when the outputs $\y^{(i)}$ are noisy, one can solve the least-squares problem
$\norm{\Ymat - \X\tenmatgen{\Hten^{(l)}}{l,1}}^2_F$ to approximate the Hankel tensors; we will empirically evaluate this approach 
in Section~\ref{sec:xp} and we defer its theoretical analysis  in the noisy setting to future work.

\subsection{\rev{Leveraging the low rank structure of the Hankel tensors}}
While the least-squares method is sufficient to obtain the theoretical guarantees of Theorem~\ref{thm:learning-2RNN}, it does not leverage
the low rank structure of the Hankel tensors $\Hten^{(L)}$, $\Hten^{(2L)}$ and $\Hten^{(2L+1)}$. We now  propose three alternative recovery
methods to leverage this structure, whose  sample efficiency will be assessed in a simulation study in Section~\ref{sec:xp}~(deriving improved sample
complexity guarantees using these methods is left for future work). In the noiseless setting, we first propose to replace solving the linear 
system $\Ymat = \X\tenmatgen{\Hten^{(l)}}{l,1}$ with a nuclear norm minimization problem~(see line~\ref{alg.line.nucnorm} of Algorithm~\ref{alg:2RNN-SL}), thus leveraging the
fact that $\tenmatgen{\Hten^{(l)}}{\ceil{l/2},l-\ceil{l/2} + 1}$ is potentially of low matrix rank. We also propose to use iterative hard thresholding~(IHT)~\cite{jain2010guaranteed}
and its tensor counterpart TIHT~\cite{rauhut2017low}, which are based on the classical projected gradient descent algorithm and have shown to be
robust to noise in practice. These two methods are implemented in lines~\ref{alg.line.iht.start}-\ref{alg.line.iht.end} of Algorithm~\ref{alg:2RNN-SL}. There,
the \texttt{project} method either projects $\tenmatgen{\Hten^{(l)}}{\ceil{l/2},l-\ceil{l/2} + 1}$ onto the manifold of low rank matrices 
using SVD~(IHT) or projects  $\Hten^{(l)}$ onto the manifold of tensors with TT-rank $R$~(TIHT).

\input{xp_figures.tex}

\rev{
The low rank structure of the Hankel tensors can also be leveraged to improve the scalability of the learning algorithm.
One can check that the computational complexity of Algorithm~\ref{alg:2RNN-SL} is exponential in the maximum sequence length: indeed,
building the matrix $\X$ in line~2 is already in $\bigo{N_ld^l}$, where $l$ is in turn equal to $L,\ 2L$ and $2L+1$.
Focusing on the TIHT recovery method, a careful analysis shows that the computational complexity of the algorithm
is in 
$$\bigo{d^{2L+1}\left(p(TN+R) +R^2\right) + TL\max(p,d)^{2L+3}}, $$
where $N=\max(N_L,N_{2L},N_{2L+1})$ and $T$ is the number of iterations of the loop on line~\ref{alg.line.iht.start}.
Thus, in its present form, our approach cannot scale to high dimensional inputs and long sequences. However, one can  leverage
the low tensor train rank structure of the Hankel tensors to circumvent this issue:
by storing 
both the estimates of the Hankel tensors $\Hten^{(l)}$ and the matrices $\X$ in TT format~(with decompositions of ranks $R$ and $N$ respectively),
all the operations needed to implement Algorithm~\ref{alg:2RNN-SL} with the TIHT recovery method can be performed in time $\bigo{T(N+R)^3(Ld + p)}$~(more details can be found in the supplementary
material). By leveraging the tensor train structure, one can thus lift the dependency on $d^{2L+1}$ by paying the price of an increased cubic complexity 
in the number of examples $N$ and the number of states $R$. While the
dependency on the number of states is not a major issue~($R$ should be negligible w.r.t. $N$), the dependency on $N^3$ can quickly become prohibitive for realistic application scenario. 
Fortunately, this issue can  be
dealt with by using mini-batches of training data for the gradient updates on line~\ref{alg.line.iht.gradient} instead of the whole dataset $D_l$, in which case the overall complexity
of Algorithm~\ref{alg:2RNN-SL} becomes $\bigo{T(M+R)^3(Ld + p)}$ where $M$ is the mini-batch size~(the overall algorithm in TT format is summarized in Algorithm~\ref{alg:2RNN-SL-TT} in the supplementary material).
}

\input{xp}

\section{Conclusion and Future Directions}

We proposed the first provable learning algorithm for second-order RNNs with linear activation functions:
we showed that linear 2-RNNs are a natural extension of vv-WFAs to the setting of input sequences of \emph{continuous vectors}~(rather than
discrete symbol) and we extended the vv-WFA spectral learning
algorithm to this setting. We believe that the results presented in this paper open a number of exciting and promising research directions on both the
theoretical and practical perspectives. We first plan to use the spectral learning estimate as a starting point for
gradient based methods to train non-linear 2-RNNs. More precisely, linear 2-RNNs can be thought of as 2-RNNs using LeakyRelu activation functions with negative slope $1$, therefore one could use 
a linear 2-RNN as initialization before gradually reducing the negative slope parameter during training. The extension of the spectral method to linear 2-RNNs also
opens the door to scaling up the classical spectral algorithm to problems with large discrete alphabets~(which is a known caveat of the spectral algorithm for WFAs) since
it allows one to use low dimensional embeddings of large vocabularies~(using \eg word2vec or latent semantic analysis). From the theoretical perspective, we plan on 
deriving learning guarantees for  linear 2-RNNs in the noisy setting~(\eg using the PAC learnability framework). Even though it is intuitive that such guarantees should hold~(given
the continuity of all operations used in our algorithm), we believe that such an analysis may entail results of independent interest. In particular, analogously to the
matrix case  studied in~\cite{cai2015rop}, obtaining rate optimal convergence rates for the recovery of the low TT-rank Hankel tensors from rank one measurements is an interesting
direction; such a result could for example allow one to improve the generalization bounds provided in~\cite{balle2012spectral} for spectral learning of 
general WFAs. 

%
%
%

\newpage
\subsubsection*{Acknowledgements} This work was done while G. Rabusseau was an IVADO postdoctoral scholar at McGill University. 
{
\bibliographystyle{plain}
\bibliography{main.bib}
}

\clearpage
\newpage
\onecolumn

\appendix

\input{supmat_content}

\end{document}

%% file: macros.tex
 \usepackage{relsize}

\usepackage{etoolbox}
\usepackage{xparse}

\renewcommand{\enspace}{}
\newcommand{\R}{\Rbb}
\newcommand{\e}{\ten}

\renewcommand{\vec}[1]{\ensuremath{\mathbf{#1}}}
\newcommand{\vecs}[1]{\ensuremath{\mathbf{\boldsymbol{#1}}}}
\newcommand{\mat}[1]{\ensuremath{\mathbf{#1}}}

\newcommand{\ten}[1]{\mat{\ensuremath{\boldsymbol{\mathcal{#1}}}}}

\newcommand{\ttm}[1]{\times_{#1}}
\newcommand{\ttv}[1]{\bullet_{#1}}
\newcommand{\tenmat}[2]{\ten{#1}_{(#2)}}
\newcommand{\tenmatpar}[2]{(#1)_{(#2)}}
\newcommand{\tenmatgen}[2]{{(#1)}_{\langle\!\langle #2\rangle\!\rangle}}

\newcommand{\TT}[1]{\llbracket #1 \rrbracket}

\newcommand{\A}{\mat{A}}
\newcommand{\B}{\mat{B}}

\newcommand{\Hten}{\ten{H}}
\newcommand{\T}{\ten{T}}

\newcommand{\G}{\ten{G}}
\newcommand{\Y}{\ten{Y}}

\renewcommand{\H}{\mat{H}}

\newcommand{\X}{\mat{X}}
\newcommand{\I}{\mat{I}}

\newcommand{\M}{\mat{M}}

\renewcommand{\P}{\mat{P}}
\renewcommand{\S}{\mat{S}}
\renewcommand{\v}{\vec{v}}

\newcommand{\x}{\vec{x}}

\newcommand{\y}{\vec{y}}

\newtheorem*{theorem*}{Theorem}
\newtheorem*{corollary*}{Corollary}%
\newtheorem*{proposition*}{Proposition}%
\ifcsdef{theorem}{}{
\newtheorem{theorem}{Theorem}%
}
\newtheorem*{pbm*}{Problem}%
\newtheorem*{algo*}{Algorithm}%

\newcommand{\kron}{\otimes}

\DeclareMathOperator*{\argmin}{arg\,min}
\newcommand{\vectorize}[1]{\mathrm{vec}(#1)}

\DeclareMathOperator*{\rank}{rank}
\newcommand{\norm}[1]{\|#1\|}

\newcommand{\bigo}[1]{\mathcal{O}\left(#1\right)}

\newcommand{\pinv}{^\dagger}
\newcommand{\inv}{^{-1}}
\newcommand{\invtop}{^{-\top}}

\newcommand{\nstates}{n}

\newcommand{\szerosymbol}{\alpha}
\newcommand{\szero}{\vecs{\szerosymbol}}

\newcommand{\sinfsymbol}{\omega}
\newcommand{\sinf}{\vecs{\sinfsymbol}}

\DeclareDocumentCommand{\wa}{  O{A} O{\szero} O{\sinf} }%
{(#2,\{\mat{#1}^\sigma\}_{\sigma\in\Sigma},#3)}
\DeclareDocumentCommand{\waR}{  O{A} O{\Rbb^\nstates} O{\szero} O{\sinf} }%
{(#2,#3,\{\mat{#1}^\sigma\}_{\sigma\in\Sigma},#4)}

\newcommand{\vvsinfsymbol}{\Omega}
\newcommand{\vvsinf}{\vecs{\vvsinfsymbol}}
\DeclareDocumentCommand{\vvwa}{  O{A} O{\szero} O{\vvsinf} }%
{(#2,\{\mat{#1}^\sigma\}_{\sigma\in\Sigma},#3)}

\newcommand{\tzerosymbol}{\alpha}
\newcommand{\tzero}{\vecs{\tzerosymbol}}

\newcommand{\tinfsymbol}{\omega}
\newcommand{\tinf}{\vecs{\tinfsymbol}}

\DeclareDocumentCommand{\wta}{ O{T} O{\Rbb^\nstates} O{\tzero} O{\tinf} O{\Fcal}}%
{(#2,#3,\{\ten{#1}^g\}_{g\in #5_{\geq 1}},\{#4^\sigma\}_{\sigma\in #5_0})}

\DeclareDocumentCommand{\trees}{g}{\IfNoValueTF{#1}{\mathfrak{T}}{\mathfrak{T}_{#1}}}
\DeclareDocumentCommand{\contexts}{g}{\IfNoValueTF{#1}{\mathfrak{C}}{\mathfrak{C}_{#1}}}

\newcommand{\gwmprod}{\diamond}

\DeclareDocumentCommand{\gwm}{  O{M} O{\Fbb^\nstates}}{(#2, \{\ten{#1}^x\}_{x\in\Sigma})}
\DeclareDocumentCommand{\gwmcirc}{  O{M} O{\Rbb^\nstates}}{(#2, \{\mat{#1}^\sigma\}_{\sigma\in\Sigma})}
\DeclareDocumentCommand{\dgwm}{ O{M} O{\Fbb^\nstates}}{(#2, \{\ten{#1}^x\}_{x\in\Sigma},\gwmprod)}

\usepackage{pgffor}
\foreach \x in {A,...,Z}{%
\expandafter\xdef\csname \x bb\endcsname{\noexpand\ensuremath{\noexpand\mathbb{\x}}}
}

\foreach \x in {A,...,Z}{%
\expandafter\xdef\csname \x cal\endcsname{\noexpand\ensuremath{\noexpand\mathcal{\x}}}
}

\usepackage{pgffor}
\foreach \x in {A,...,Z}{%
\expandafter\xdef\csname \x ten\endcsname{\noexpand\ensuremath{\noexpand\ten{\x}}}
}

\foreach \x in {A,...,Z}{%
\expandafter\xdef\csname \x mat\endcsname{\noexpand\ensuremath{\noexpand\mat{\x}}}
}

\foreach \x in {A,...,Z}{%
\expandafter\xdef\csname \x vec\endcsname{\noexpand\ensuremath{\noexpand\mat{\x}}}
}

%% file: tikz_tensor_networks.tex
\tikzset{tensor/.style = {shape          = circle,
                                 ball color     = orange,
                                 text           = black,
                                 inner sep      = 2pt,
                                 outer sep      = 0pt,
                                 minimum size   = 18 pt}}
  \tikzset{edge/.style   = {thick,
                                 double          = orange,
                                 double distance = 1pt}}
  \tikzset{dim/.style =   {draw,
                                  fill           = yellow,
                                  text           = black,
                                  scale=0.7}}
                               
  \tikzset{port/.style =   {draw = none,
                                  scale=0.6}}   
                                  
 \DeclareDocumentCommand{\edgeports}%
{m m m m m m m o}%
{
\IfNoValueTF {#8}
{ \draw[edge] (#1) -- (#4) node[port,pos=0.05, #3] {#2} node[dim,pos=0.5] {#7} node[port,pos=0.95, #6] {#5} ;}
{ \draw[edge,draw=none] (#1) -- (#4) node[port,pos=0.05, #3]  {#2}  node[port,pos=0.95, #6] {#5} ;
  \draw[edge] (#1) to[#8] node[dim,pos=0.5,below=-0.15cm] {#7} (#4);
  }
 }                                 
\newcommand{\simpleedgeports}[6]{%
\draw[edge] (#1) -- (#4) node[port,pos=0.05, #3] {#2} node[port,pos=0.95, #6] {#5} ;}

%% file: xp_figures.tex
\begin{figure*}[ht]
\vspace*{-0.4cm}
\begin{center}
\hspace*{-1cm}\includegraphics[width=0.9\textwidth]{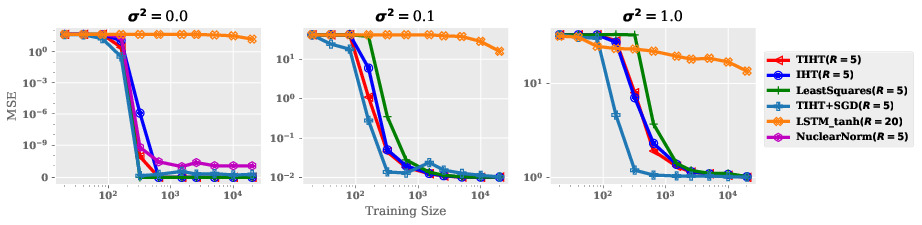}
\end{center}
\vspace*{-0.5cm}
\caption{ Average MSE as a function of the training set size for the first experiment~(learning a random linear 2-RNN) for different values of output noise.}
\label{fig:xp.random}
\end{figure*}

\begin{figure*}[ht]
\begin{center}
\vspace*{-0.25cm}
\hspace*{-1cm}\includegraphics[width=0.9\textwidth]{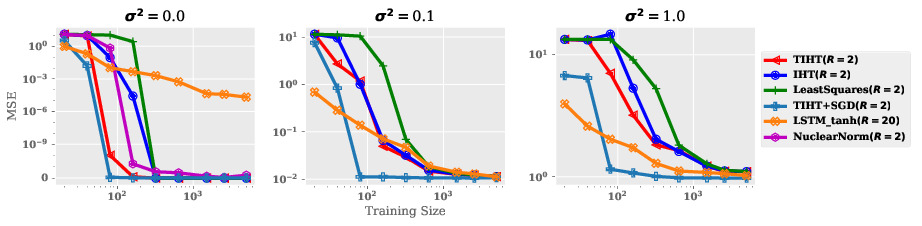}
\end{center}
\vspace*{-0.5cm}
\caption{ Average MSE as a function of the training set size for the second experiment~(learning a simple arithmetic function) for different values of output noise.}
\label{fig:xp.addition}
\end{figure*}

%% file: xp.tex
\section{Experiments}\label{sec:xp}



\rev{
In this section, we perform experiments\footnote{\url{https://github.com/litianyu1993/learning_2RNN}} on two toy examples to compare how the choice of the recovery method~(\texttt{LeastSquares}, \texttt{NuclearNorm}, \texttt{IHT} and \texttt{TIHT}) affects the sample efficiency of Algorithm~\ref{alg:2RNN-SL}. Additional experiments on real data can be found in Appendix~\ref{app:real}.
We also include comparisons with RNNs with long short term memory~(LSTM) units~\cite{hochreiter1997long} and report the performances
obtained by  refining  the solution returned by our algorithm~(with the TIHT recovery method) using stochastic gradient descent~(\texttt{TIHT+SGD}). 
}



We perform two experiments. In the first one, we randomly generate a linear 2-RNN with $5$ units computing a function $f:\R^3\to\R^2$ by 
drawing the entries of all parameters  $(\h_0,\Aten,\vvsinf)$  independently  from the normal distribution
$\Ncal(0,0.2)$.
The training data consists of $3$ independently drawn sets 
$D_l = \{((\x^{(i)}_1,\x_2^{(i)},\cdots,\x_l^{(i)}),\y^{(i)}) \}_{i=1}^{N_l}\subset(\R^d)^l\times \R^p$ for $l\in\{L,2L,2L+1\}$ with $L=2$,
where  each $\x^{(i)}_j\sim\Ncal(\vec{0},\I)$ and where the outputs can be noisy, \ie 
$\y^{(i)} = f(\x^{(i)}_1,\x_2^{(i)},\cdots,\x_l^{(i)}) + \vecs{\xi}^{(i)}$ where $\vecs{\xi}^{(i)}\sim\Ncal(0,\sigma^2)$ for some noise variance $\sigma^2$. 
In the second experiment, the goal is to learn a simple arithmetic function computing the sum of the running differences between the two components of a sequence
of $2$-dimensional vectors, \ie $f(\x_1,\cdots,\x_k) = \sum_{i=1}^k \v^\top\x_i$ where $\v^\top = (-1\ \ 1)$. The $3$ training datasets are generated 
using the same process as above and a constant entry equal to one is added to all the input vectors to encode a bias 
term~(one can check that the resulting function can be computed by a linear 2-RNN with $2$ hidden units). 

We run the experiments for different sizes of training data ranging from $N=20$ to $N=20,000$~(we set $N_L=N_{2L}=N_{2L+1}=N$) 
and we compare the different methods in terms of mean squared error~(MSE)
 on a test set of $1,000$ sequences of length $6$ generated in the same
way as the training data~(note that the training data only contains sequences of length up to $5$). 

\rev{We report the performances of (non-linear) RNNs with a single layer of LSTM with $20$  hidden units~(with $\mathrm{tanh}$ activation 
functions\footnote{We also tried training LSTMs with linear recurrent activation functions on the two tasks but they always performed worse than non-linear ones.}) and one fully-connected output layer, trained using
the Adam optimizer~\cite{kingma2014adam} with learning rate 0.001. We also use Adam with learning rate 0.001 to refine the models returned by TIHT with stochastic gradient
descent SGD (we tried directly training a linear 2-RNN from random initializations using SGD as well but this approach always failed to return a good model). 
The IHT/TIHT methods sometimes returned aberrant models~(due to numerical instabilities), we used the following scheme to circumvent this issue: 
when the training MSE of the hypothesis was greater than the one of the zero function,
the zero function was returned instead (we applied this scheme to all other methods in the experiments).}


The results are reported in Figure~\ref{fig:xp.random} and~\ref{fig:xp.addition} where we see that all recovery methods of Algorithm~\ref{alg:2RNN-SL} lead to consistent estimates of the
target function given enough training data. This is the case even in the presence of noise~(in which case more samples are needed to achieve the same accuracy, as expected).
We can also see that \texttt{IHT} and \texttt{TIHT} are overall more sample efficient than the other methods~(especially with noisy data), showing
that taking the low rank structure of the Hankel tensors into account is profitable. Moreover, \texttt{TIHT} tends to perform better than its matrix counterpart, 
confirming our intuition that leveraging the tensor train 
structure is beneficial. 
\rev{While LSTMs obtain good performances on the addition task, they struggle to recover the random linear 2-RNN in the first task~(despite our efforts at hyper-parameter tuning and architecture search). 
In the meantime, refining the TIHT models using SGD almost always leads to significant improvements~(especially under the noisy setting), matching or outperforming the
performances of RNNs on the two tasks.
Lastly, we show the effect of rank mis-specification in Figure~\ref{fig:xp.misrank}: as one can expect, when the rank parameter $R$ 
is over-estimated Algorithm~\ref{alg:2RNN-SL} still converges to the target function but it requires  more samples~(when the rank parameter was underestimated all  algorithms did not learn at all). 
}

\begin{figure}[t]
\begin{center}
\hspace*{0cm}\includegraphics[width=0.5\textwidth]{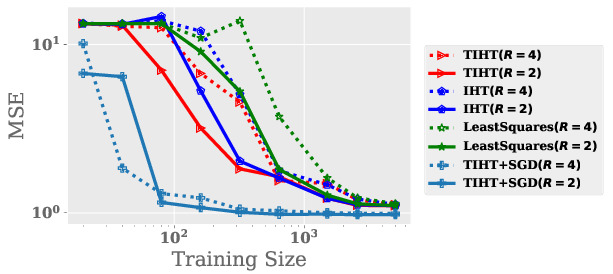}
\end{center}
\vspace*{-0.5cm}
\caption{ Comparison between different rank settings in terms of average MSE for the second experiment~(learning a simple arithmetic function) in the noisy
setting~($\sigma^2=1$).}
\label{fig:xp.misrank}
\end{figure}

%% file: supmat_content.tex
\section*{\centering \LARGE Connecting Weighted Automata and Recurrent Neural Networks through Spectral Learning \\ \vspace*{0.3cm}(Supplementary Material)}

\section{Proofs}
\subsection{Proof of Theorem~\ref{thm:2RNN-vvWFA}}
\begin{theorem*}
Any function that can be computed by a vv-WFA with $n$ states can be computed by a linear 2-RNN with $n$ hidden units.
Conversely, any function that can be computed by a linear 2-RNN with $n$ hidden units on sequences of one-hot vectors~(\ie canonical basis 
vectors) can be computed by a WFA with $n$ states.

More precisely, the WFA $A=\vvwa$ with $n$ states and the linear 2-RNN $M=(\szero,\Aten,\vvsinf)$ with
$n$ hidden units, where $\Aten\in\R^{n\times \Sigma \times n}$ is defined by $\Aten_{:,\sigma,:}=\A^\sigma$ for all $\sigma\in\Sigma$, are
such that
$f_A(\sigma_1\sigma_2\cdots\sigma_k) = f_M(\x_1,\x_2,\cdots,\x_k)$ for all sequences of input symbols $\sigma_1,\cdots,\sigma_k\in\Sigma$,
where for each $i\in[k]$ the input vector $\x_i\in\R^\Sigma$ is
the one-hot encoding of the symbol $\sigma_i$.
\end{theorem*}
\begin{proof}
We first show by induction on $k$ that, for any sequence $\sigma_1\cdots\sigma_k\in\Sigma^*$, the hidden state $\h_k$ computed by $M$~(see
Eq.~\eqref{eq:2RNN.definition})
on the corresponding one-hot encoded sequence $\x_1,\cdots,\x_k\in\R^d$
satisfies $\h_k = (\A^{\sigma_1}\cdots\A^{\sigma_k})^\top\szero$. The case $k=0$
is immediate. Suppose the result true for sequences of length up to $k$. One can check easily check that $\Aten\ttv{2}\x_i = \A^{\sigma_i}$
for any index $i$. Using the induction hypothesis it then follows that
\begin{align*}
\h_{k+1} &= \Aten \ttv{1}\h_k \ttv{2} \x_{k+1} = \A^{\sigma_{k+1}}\ttv{1} \h_k = (\A^{\sigma_{k+1}})^\top \h_k\\
&= (\A^{\sigma_{k+1}})^\top (\A^{\sigma_1}\cdots\A^{\sigma_k})^\top\szero = (\A^{\sigma_1}\cdots\A^{\sigma_{k+1}})^\top\szero .
\end{align*} 
To conclude, we thus have
\begin{equation*}
f_M(\x_1,\x_2,\cdots,\x_k) = \vvsinf\h_{k} = \vvsinf(\A^{\sigma_1}\cdots\A^{\sigma_{k}})^\top\szero = f_A(\sigma_1\sigma_2\cdots\sigma_k).\qedhere
\end{equation*}
\end{proof}

\subsection{Proof of Theorem~\ref{thm:2RNN-SL}}

\begin{theorem*}
Let $f:(\R^d)^*\to \R^p$ be a function computed by a minimal linear $2$-RNN  with $n$ hidden units and let
$L$ be an integer such that $\rank(\tenmatgen{\Hten^{(2L)}_f}{L,L+1}) = n$.

Then, for any $\P\in\R^{d^L\times n}$ and $\S\in\R^{n\times d^Lp}$ such that $\tenmatgen{\Hten^{(2L)}_f}{L,L+1} = \P\S$, the
linear 2-RNN $M=(\szero,\Aten,\vvsinf)$ defined by
$$\szero = (\S\pinv)^\top\tenmatgen{\Hten^{(L)}_f}{L+1},\ \ \ \ \Aten = (\tenmatgen{\Hten^{(2L+1)}_f}{L,1,L+1})\ttm{1}\P\pinv\ttm{3}(\S\pinv)^\top,\ \ \ \ 
\vvsinf^\top = \P\pinv\tenmatgen{\Hten^{(L)}_f}{L,1}$$
is a minimal linear $2$-RNN computing $f$.
\end{theorem*}
\begin{proof}
Let $\P\in\R^{d^L\times n}$ and $\S\in\R^{n\times d^Lp}$ be such that $\tenmatgen{\Hten^{(2L)}_f}{L,L+1} = \P\S$
Define the tensors 
$$\Pten^* = \TT{\Aten^\star\ttv{1}\szero^\star, \underbrace{\Aten^\star, \cdots, \Aten^\star}_{L-1\text{ times}}, \I_n}\in\R^{d\times\cdots\times d\times n}\ \ \ \  
\text{ and }\ \ \ \ 
\Sten^* = \TT{\I_n,\underbrace{\Aten^\star, \cdots, \Aten^\star}_{L\text{ times}}, \vvsinf^\star}\in\R^{n\times d\times\cdots\times d\times p}$$ 
of order $L+1$ and $L+2$ respectively, and let $\P^\star = \tenmatgen{\Pten^*}{l,1} \in\R^{d^l\times n}$ 
and $\S = \tenmatgen{\Sten^*}{1,L+1}  \in\R^{n\times d^lp}$. Using the identity $\Hten^{(j)}_f = \TT{\Aten\ttv{1}\szero, \underbrace{\Aten, \cdots, \Aten}_{j-1\text{ times}}, \vvsinf^\top}$ for
any $j$, one can easily check the following identities:
\begin{gather*}
\tenmatgen{\Hten^{(2L)}_f}{L,L+1} = \P^\star\S^\star,\ \  \ \ \tenmatgen{\Hten^{(2L+1)}_f}{L,1,L+1}= \Aten^\star \ttm{1} \P^\star \ttm{3} (\S^\star)^\top,\\
\tenmatgen{\Hten^{(L)}_f}{L,1} = \P^\star(\vvsinf^\star)^\top, \ \ \ \ \ \ \
\tenmatgen{\Hten^{(L)}_f}{L+1} = (\S^\star)^\top\szero.
\end{gather*}

Let $\M = \P\pinv\P^\star$. We will show that $\szero = \M\invtop\szero^\star$, $\Aten = \Aten^\star \ttm{1}\M\ttm{3}\M\invtop$ and
$\vvsinf = \M\vvsinf^\star$, which will entail the results since linear 2-RNN are invariant under change of basis~(see Section~\ref{sec:prelim}). First observe that $\M\inv = \S^\star\S\pinv$. Indeed,
we have
$\P\pinv\P^\star\S^\star\S\pinv = \P\pinv\tenmatgen{\Hten^{(2l)}_f}{l,l+1}\S\pinv = \P\pinv\P\S\S\pinv = \I$
where we used the fact that $\P$~(resp. $\S$) is of full column rank~(resp. row rank) for the last equality. 

The following derivations then follow from basic tensor algebra:
\begin{align*}
\szero 
&= 
(\S\pinv)^\top\tenmatgen{\Hten^{(L)}_f}{L+1} 
=
(\S\pinv)^\top (\S^\star)^\top\szero
=
(\S^\star\S\pinv)^\top
=
\M\invtop\szero^\star,\\
\ \\
\Aten 
&= 
(\tenmatgen{\Hten^{(2L+1)}_f}{L,1,L+1})\ttm{1}\P\pinv\ttm{3}(\S\pinv)^\top\\
&=
(\Aten^\star \ttm{1} \P^\star \ttm{3} (\S^\star)^\top)\ttm{1}\P\pinv\ttm{3}(\S\pinv)^\top\\
&=
\Aten^\star \ttm{1} \P\pinv\P^\star \ttm{3} (\S^\star\S\pinv)^\top =  \Aten^\star \ttm{1}\M\ttm{3}\M\invtop,\\
\ \\
\vvsinf^\top 
&= 
\P\pinv\tenmatgen{\Hten^{(L)}_f}{L,1}
=
\P\pinv\P^\star(\vvsinf^\star)^\top
=
\M\vvsinf^\star,
\end{align*}
which concludes the proof.
\end{proof}

\subsection{Proof of Theorem~\ref{thm:learning-2RNN}}

\begin{theorem*}
Let $(\h_0,\Aten,\vvsinf)$ be a minimal linear 2-RNN with $n$ hidden units computing a function $f:(\R^d)^*\to \R^p$, and let $L$ be an integer\footnote{Note that the theorem can be adapted if such an integer $L$ does not exists~(see supplementary material).}
such that $\rank(\tenmatgen{\Hten^{(2L)}_f}{L,L+1}) = n$.

Suppose we have access to $3$ datasets
$D_l = \{((\x^{(i)}_1,\x_2^{(i)},\cdots,\x_l^{(i)}),\y^{(i)}) \}_{i=1}^{N_l}\subset(\R^d)^l\times \R^p$ for $l\in\{L,2L,2L+1\}$ 
where the entries of each $\x^{(i)}_j$ are drawn independently from the standard normal distribution and where each
$\y^{(i)} = f(\x^{(i)}_1,\x_2^{(i)},\cdots,\x_l^{(i)})$.

Then, whenever $N_l \geq d^l$ for each $l\in\{L,2L,2L+1\}$,
the linear 2-RNN $M$ returned by Algorithm~\ref{alg:2RNN-SL} with the least-squares method satisfies $f_M = f$ with probability one.
\end{theorem*}
\begin{proof}
We just need to show for each $l\in \{L,2L,2L+1\}$ that, under the hypothesis of the Theorem, the Hankel tensors $\hat{\Hten}^{(l)}$ computed in line~\ref{alg.line.lst-sq} of
Algorithm~\ref{alg:2RNN-SL} are equal to the true Hankel tensors $\Hten^{(l)}$ with probability one. Recall that these tensors are computed by solving the least-squares
problem
$$\hat{\Hten}^{(l)} = \argmin_{T\in \R^{d\times\cdots\times d\times p}} \norm{\X\tenmatgen{\T}{l,1} - \Ymat}_F^2$$
where $\X\in\R^{N_l\times d_l}$ is the matrix  with rows $\x^{(i)}_1\kron\x_2^{(i)}\kron\cdots\kron\x_l^{(i)}$ for each $i\in[N_l]$. Since  $\X\tenmatgen{\Hten^{(l)}}{l,1} = \Ymat$ and since the solution
of the least-squares problem is unique as soon as $\X$ is of full column rank, we just need to show that this is the case with probability one
when the entries of the vectors $\x^{(i)}_j$ are drawn at random from a standard normal distribution. The result will  then directly follow
by applying Theorem~\ref{thm:2RNN-SL}.

We will show that the set 
$$\Scal = \{ (\x_1^{(i)},\cdots, \x_l^{(i)}) \mid \ i\in[N_l],\ dim(span(\{ \x^{(i)}_1\kron\x_2^{(i)}\kron\cdots\kron\x_l^{(i)} \})) < d^l\} $$
has Lebesgue measure $0$ in $((\R^d)^{l})^{N_l}\simeq \R^{dlN_l}$ as soon as $N_l \geq d^l$, which will imply that it has probability $0$ under any continuous probability, hence
the result. For any  $S=\{(\x_1^{(i)},\cdots, \x_l^{(i)})\}_{i=1}^{N_l}$, we denote  by $\X_S\in\R^{N_l\times d^l}$  the matrix  with rows $\x^{(i)}_1\kron\x_2^{(i)}\kron\cdots\kron\x_l^{(i)}$.
One can easily check that $S\in\Scal$ if and only if $\X_S$ is of rank strictly less than $d^l$, which is equivalent to the determinant of 
$\X_S^\top\X_S$ being equal to $0$. Since this determinant is a polynomial in the entries of the vectors $\x_j^{(i)}$, $\Scal$ is an algebraic
subvariety of $\R^{dlN_l}$.
It is then easy to check that the polynomial $det(\X_S^\top\X_S)$ is not uniformly 0 when $N_l \geq d^l$. Indeed, 
it suffices to choose the vectors $\x_j^{(i)}$ such that the family  $(\x^{(i)}_1\kron\x_2^{(i)}\kron\cdots\kron\x_l^{(i)})_{n=1}^{N_l}$ spans the whole space 
$\R^{d^l}$~(which is possible since we can choose arbitrarily any of the $N_l\geq d^l$ elements of this family),  hence the result. 
In conclusion, $\Scal$ is a proper algebraic subvariety of $\R^{dlN_l}$ and hence has Lebesgue
measure zero~\cite[Section 2.6.5]{federer2014geometric}.

\end{proof}

\section{Lifting the simplifying assumption}\label{app:lift}
We now show how all our results still hold when there does not exist an $L$ such that $\rank(\tenmatgen{\Hten^{(2L)}_f}{L,L+1}) = n$.
Recall that this simplifying assumption followed from assuming that the sets $\Pcal=\Scal=[d]^L$ form a complete basis for the function
$\tilde{f}:[d]^*\to \R^p$ defined by $\tilde{f}(i_1i_2\cdots i_k) = f(\e_{i_1},\e_{i_2},\cdots,\e_{i_k})$. 
We first show that there always exists an integer $L$ such that $\Pcal=\Scal=\cup_{i\leq L} [d]^i$ forms a complete basis for $\tilde{f}$.
 Let $M = (\szero^\star,\Aten^\star,\vvsinf^\star)$  be a linear 2-RNN with $n$ hidden units computing
$f$~(\ie such that $f_M=f$). It follows from Theorem~\ref{thm:2RNN-vvWFA} and from the discussion  at the beginning of Section~\ref{subsec:SL-2RNN}
that there exists a vv-WFA computing $\tilde{f}$ and it is easy to check that $\rank(\tilde{f}) = n$.  This implies $\rank(\tenmatpar{\Hten_f}{1}) = n$ by
Theorem~\ref{thm:fliess-vvWFA}. Since  $\Pcal=\Scal=\cup_{i\leq l} [d]^i$ converges to $[d]^*$ as $l$ grows to infinity, there exists an $L$ such that 
the finite sub-block $\tilde{\Hten}_f \in \R^{\Pcal\times\Scal\times p}$ of $\Hten_f\in \R^{[d]^*\times[d]^*\times p}$
satisfies $\rank(\tenmatpar{\tilde{\Hten}_f}{1}) = n$, \ie such that
$\Pcal=\Scal=\cup_{i\leq L} [d]^i$ forms a complete basis for $\tilde{f}$. 

Now consider the finite sub-blocks $\tilde{\Hten}^{+}_f\in \R^{\Pcal \times [d] \times \Scal \times p}$ and  $\tilde{\H}^{-}_f\in \R^{\Pcal  \times p}$ of $\Hten_f$  defined by
$$(\tilde{\Hten}^{+}_f)_{u,i,v,:}=\tilde{f}(uiv),\ \ \ \text{and} (\tilde{\H}^{-}_f)_{u,:}=f(u)$$
for any $u\in \Pcal=\Scal$ and any $i\in [d]$. One can check that Theorem~\ref{thm:2RNN-SL} holds by replacing \emph{mutatis mutandi}
$\tenmatgen{\Hten^{(2L)}_f}{L,L+1}$ by $\tenmatpar{\tilde{\Hten}_f}{1}$, $\tenmatgen{\Hten^{(2L+1)}_f}{L,1,L+1}$ by  $\tilde{\Hten}^{+}_f$, $\tenmatgen{\Hten^{(L)}_f}{L,1}$ by $\tilde{\H}^{-}_f$
and $\tenmatgen{\Hten^{(L)}_f}{L+1}$ by $\vectorize{\tilde{\H}^{-}_f}$.

To conclude, it suffices to observe that both $\tilde{\Hten}^{+}_f$ and  $\tilde{\H}^{-}_f$ can be constructed
from the entries for the tensors $\Hten^{(l)}$ for $1\leq l \leq 2L+1$, which can be recovered~(or estimated in the noisy setting) using
the techniques described in Section~\ref{subsec:Hankel.tensor.recovery}~(corresponding to lines \ref{alg.firstline.forloop}-\ref{alg.lastline.forloop} of
Algorithm~\ref{alg:2RNN-SL}).

We thus showed that linear 2-RNNs can be provably learned even when  there does not exist an $L$ such that $\rank(\tenmatgen{\Hten^{(2L)}_f}{L,L+1}) = n$. In this
setting, one needs to estimate enough of the tensors $\Hten^{(l)}$ to reconstruct a complete sub-block $\tilde{\Hten}_f$ of the Hankel tensor 
$\Hten$~(along with the corresponding tensor $\tilde{\Hten}^{+}_f$ and matrix $\tilde{\H}^{-}_f$)
and recover the linear 2-RNN by applying Theorem~\ref{thm:2RNN-SL}. In addition, one needs to have access to sufficiently large datasets 
$D_l$ for each $l\in [2L+1]$ rather than only the three datasets mentioned in Theorem~\ref{thm:learning-2RNN}.  However the data requirement remains the
same in the case where we assume that each of the datasets $D_l$ is constructed from a unique training  dataset 
$S =\{((\x^{(i)}_1,\x_2^{(i)},\cdots,\x_T^{(i)}),(\y^{(i)}_1,\y^{(i)}_2,\cdots,\y^{(i)}_T)) \}_{i=1}^{N}$
of input/output sequences.

\section{Leveraging the tensor train structure for computational efficiency}
The overall learning algorithm using the TIHT recovery method in TT format is summarized in Algorithm~\ref{alg:2RNN-SL-TT}. The key ingredients to improve the complexity of Algorithm~\ref{alg:2RNN-SL} are (i) to estimate the gradient using mini-batches of data and  (ii) to directly use the TT format to represent and perform operations on the tensors $\Hten^{(l)}$ and the tensors $\Xten^{(l)}\in\R^{M\times d\times \cdots \times d}$ defined by
\begin{equation}\label{eq:Xten}
 \Xten_{i,:,\cdots,:}= \x^{(i)}_1\kron\x_2^{(i)}\kron\cdots\kron\x_l^{(i)}\ \  \ \text{for }i\in[M]
\end{equation}
where $M$ is the size of a mini-batch of training data~($\Hten^{(l)}$ is of TT-rank $R$ by design and it can easily be shown that $\Xten^{(l)}$ is of TT-rank at most $M$, cf. Eq.~\eqref{eq:XtenTT}). 
Then, all the operations of the algorithm can be expressed in terms of these tensors and performed efficiently in TT format.
More precisely, the products and sums needed to compute the gradient update on line~\ref{algtt.line.iht.gradient} can be performed in~$\bigo{(R+M)^2(ld+p)+(R+M)^3d}$. After the gradient update, the tensor $\Hten^{(l)}$ has TT-rank at most $(M+R)$ but can be efficiently projected back to a tensor of TT-rank $R$ using the tensor train rounding operation~\cite{oseledets2011tensor} in $\bigo{(R+M)^3(ld+p)}$~(which is the operation dominating the complexity of the whole algorithm).
The subsequent operations on line~\ref{line.algtt.return} can be performed efficiently in the TT format in~$\bigo{R^3d+R^2p}$~(using the method described in~\cite{klus2018tensor} to compute the pseudo-inverses of the matrices $\P$ and $\S$). 
The overall complexity of Algorithm~2 is thus in~$\bigo{T(R+M)^3(Ld+p)}$ where $T$ is the number of iterations of the inner loop.

\begin{algorithm}
   \caption{\texttt{2RNN-SL-TT}: Spectral Learning of linear 2-RNNs  \textbf{in tensor train format}}
   \label{alg:2RNN-SL-TT}
\begin{algorithmic}[1]
   \REQUIRE Three training datasets $D_L,D_{2L},D_{2L+1}$ with input sequences of length $L$, $2L$ and $2L+1$ respectively, rank $R$, learning rate $\gamma$ and
   mini-batch size $M$.
   \FOR{$l\in\{L,2L,2L+1\}$}
   \STATE Initialize all cores of the rank $R$ TT-decomposition $\Hten^{(l)} = \TT{\Gten^{(l)}_1,\cdots,\Gten^{(l)}_{l+1}} \in \R^{d\times\cdots\times d\times p}$ to $\mat{0}$.\\
   // \emph{Note that all the updates of  $\Hten^{(l)}$ stated below are in effect applied directly to the core tensors $\Gten^{(l)}_k$, i.e. the tensor $\Hten^{(l)}$ is never 
   explicitely constructed.}
   
   \REPEAT\label{algtt.line.iht.start}
   \STATE Subsample a minibatch $$\{((\x^{(i)}_1,\x_2^{(i)},\cdots,\x_l^{(i)}),\y^{(i)}) \}_{i=1}^{M}\subset(\R^d)^l\times \R^p$$ of size $M$ from $D_l$.
   \STATE Compute the rank $M$ TT-decomposition of the tensor $\Xten=\Xten^{(l)}$~(defined in Eq.~\eqref{eq:Xten}), which is given by
   \begin{equation}\label{eq:XtenTT}
       \Xten = \TT{\I_M, \Aten_1,\cdots,\Aten_l} \text{ where the cores are defined by } (\Aten_k)_{i,:,j}=\delta_{ij}\x^{(i)}_k \ \ \text{ and } \ \  (\Aten_l)_{i,:} = \x^{(i)}_k
   \end{equation}
   for all $1\leq k <l$, $i,j\in[M]$, where $\delta$ is the Kroencker symbol.
   \STATE\label{algtt.line.iht.gradient} Perform the gradient update using efficient addition and product operations in TT format~(see~\cite{oseledets2011tensor}):
   $$\tenmatgen{\Hten^{(l)}}{l,1} = \tenmatgen{\Hten^{(l)} }{l,1} + \gamma\tenmatgen{\Xten}{1,l}^\top(\Ymat - \tenmatgen{\Xten}{1,l}\tenmatgen{\Hten^{(l)} }{l,1})$$
   \STATE Project the Hankel tensor $\Hten^{(l)}$~(which is now of rank at most $R+M$) back onto the manifold of tensor of TT-rank $R$ using the TT rounding
   operation~(see again~\cite{oseledets2011tensor}):
   $$\Hten^{(l)} = \texttt{TT-rounding}(\Hten^{(l)},R)$$
   \UNTIL{convergence}\label{algtt.line.iht.end}
   \ENDFOR
   \STATE\label{algtt.line.svd} Let $\P = \tenmatgen{\TT{\Gten^{(2L)}_1,\cdots,\Gten^{(2L)}_L,\I_R}}{L,1}$ and
   $\S = \tenmatgen{\TT{\I_R,\Gten^{(2L)}_{L+1},\cdots,\Gten^{(2L)}_{2L+1}}}{1,L+1}$~(observe that
   $\tenmatgen{\Hten^{(2L)}}{L,L+1} = \P\S$ is a rank $R$ factorization).
   \STATE\label{line.algtt.return} Return the linear 2-RNN $(\h_0,\Aten,\vvsinf)$ where 
   \begin{align*}
    \szero\ &= (\S\pinv)^\top\tenmatgen{\Hten^{(L)}_f}{L+1}\\
    \Aten\ &= (\tenmatgen{\Hten^{(2L+1)}_f}{L,1,L+1})\ttm{1}\P\pinv\ttm{3}(\S\pinv)^\top\\
\vvsinf^\top &= \P\pinv\tenmatgen{\Hten^{(L)}_f}{L,1}
\end{align*}  
by performing efficient computations in TT format for the products~\cite{oseledets2011tensor} and pseudo-inverses~(see e.g. \cite{klus2018tensor}).
\end{algorithmic}
\end{algorithm}

\section{Real Data Experiment on Wind Speed Prediction}\label{app:real}
Besides the synthetic data experiments we showed in the paper, we have also conducted experiments on real data. The data that we use for the these experiments is from TUDelft\footnote{http://weather.tudelft.nl/csv/}. Specifically, we use the data from Rijnhaven station as described in~\cite{lin2016short}, which proposed a regression automata model and performed various experiments on the dataset we mentioned above. The data contains wind speed and related information at the Rijnhaven station from  2013-04-22 at 14:55:00 to 2018-10-20 at 11:40:00 and was collected every five minutes. To compare to the results in~\cite{lin2016short}, we strictly followed the data preprocessing procedure described in the paper. We use the data from 2013-04-23 to 2015-10-12 as training data and the rest as our testing data.  The paper uses SAX as a preprocessing method to discretize the data. However, as there is no need to discretize data for our algorithm, we did not perform this procedure. For our method, we set the length $L = 3$ and we use the general algorithm described in Appendix~\ref{app:lift}. We calculate hourly averages of the wind speed, and predict one/three/six hour(s) ahead, as in~\cite{lin2016short}.  For our methods we use a linear 2-RNN with 10 states. Averages over 5 runs of this experiment for one-hour-ahead, three-hour-ahead, six-hour-ahead prediction error can be found in Table~\ref{one_hour}, \ref{three_hours} and \ref{six_hours}. The results for RA, RNN and persistence are taken directly from~\cite{lin2016short}. We see that while TIHT+SGD performs slightly worse than ARIMA and RA for one-hour-ahead prediction, it outperforms all other methods for three-hours and six-hours ahead predictions~(and the superiority w.r.t. other methods increases as the prediction horizon gets longer).

\begin{table}[h]
\caption{One-hour-ahead Speed Prediction Performance Comparisons}
\centering
\begin{tabular}{c|cccccc}
Method & TIHT     & TIHT+SGD & \begin{tabular}[c]{@{}c@{}}Regression \\ Automata\end{tabular} & ARIMA & RNN & Persistence\\ \hline
RMSE   & 0.573  & 0.519   & 0.500    & \textbf{0.496 }& 0.606 & 0.508                                                    \\
MAPE   & 21.35  & 18.79    & \textbf{18.58}   & 18.74 & 24.48 & 18.61                                                    \\
MAE    & 0.412  & 0.376    & 0.363    & \textbf{0.361} & 0.471 & 0.367                                                    
\end{tabular}%
\label{one_hour}
\end{table}
\begin{table}[h]
\caption{Three-hour-ahead Speed Prediction Performance Comparisons}
\centering
\begin{tabular}{c|cccccc}
Method & TIHT     & TIHT+SGD & \begin{tabular}[c]{@{}c@{}}Regression \\ Automata\end{tabular} & ARIMA & RNN & Persistence\\ \hline
RMSE   & 0.868  & \textbf{0.854}    & 0.872  & 0.882 & 1.002 & 0.893                                                       \\
MAPE   & 33.98  & \textbf{31.70}    & 32.52 & 33.165 & 37.24 & 33.29                                                      \\
MAE    & 0.632  & \textbf{0.624}    & 0.632 &0.642 & 0.764 & 0.649                                                       
\end{tabular}
\label{three_hours}
\end{table}
\begin{table}[h]
\caption{Six-hour-ahead Speed Prediction Performance Comparisons}
\centering
\begin{tabular}{c|cccccc}
Method & TIHT   & TIHT+SGD & \begin{tabular}[c]{@{}c@{}}Regression \\ Automata\end{tabular} & ARIMA & RNN & Persistence\\ \hline
RMSE   & 1.234  & \textbf{1.145}    & 1.205  & 1.227 & 1.261 & 1.234                                                      \\
MAPE   & 49.08  & \textbf{44.88}    & 46.809  &48.02 & 47.03 & 48.11                                                      \\
MAE    & 0.940  & \textbf{0.865}    & 0.898  & 0.919 & 0.944 & 0.923            
\end{tabular}
\label{six_hours}
\end{table}